\documentclass[letterpaper]{amsart}
\usepackage[utf8]{inputenc} 
\usepackage[T1]{fontenc}    
\usepackage{url}            
\usepackage{booktabs}       
\usepackage{amsfonts}       
\usepackage{nicefrac}       
\usepackage{microtype}      
\usepackage{xcolor}         
\usepackage{amsmath,amssymb,amsxtra,amsthm}
\usepackage{tikz}
\usetikzlibrary{calc,decorations.markings}
\usetikzlibrary{decorations.pathreplacing,calligraphy}
\usetikzlibrary{arrows,positioning, shapes.symbols,shapes.callouts,patterns}

\usepackage{graphicx}
\usepackage{caption, subcaption}

\usepackage{standalone}

\usepackage{multicol}
\usepackage{enumitem}

\newcommand{\Bmath}[1]{\mbox{\bf {#1}}}

\newcommand{\Bx}{\Bmath{x}}

\newcommand{\By}{\Bmath{y}}

\def\x{{\Bmath x}}

\def\u{{\Bmath u}}

\usepackage{thmtools}
\usepackage{thm-restate}

\newtheorem{theorem}{Theorem}

\newtheorem{lemma}[theorem]{Lemma}

\theoremstyle{definition}
\newtheorem{definition}[theorem]{Definition}

\newtheorem{remark}[theorem]{Remark}

\usepackage{xcolor}
\usepackage{soul}

\newcommand{\norm}[1]{\left\lVert#1\right\rVert}

\usepackage{soul}

\title{Universality of Real Minimal Complexity Reservoir}

\author{Robert Simon Fong}
\address{School of Computer Science, University of Birmingham, Birmingham, B15 2TT, UK}
\email{r.s.fong@bham.ac.uk}
\author{Boyu Li}
\address{Department of Mathematical Sciences, New Mexico State University, Las Cruces, New Mexico, 88003, USA}
\email{boyuli@nmsu.edu}
\author{Peter Ti\v{n}o}
\address{School of Computer Science, University of Birmingham, Birmingham, B15 2TT, UK}
\email{p.tino@bham.ac.uk}

\begin{document}

\begin{abstract}
Reservoir Computing (RC) models, a subclass of recurrent neural networks, are distinguished by their fixed, non-trainable input layer and dynamically coupled reservoir, with only the static readout layer being trained. This design circumvents the issues associated with backpropagating error signals through time, thereby enhancing both stability and training efficiency. RC models have been successfully applied across a broad range of application domains. Crucially, they have been demonstrated to be universal approximators of time-invariant dynamic filters with fading memory, under various settings of approximation norms and input driving sources.

Simple Cycle Reservoirs (SCR) represent a specialized class of RC models with a highly constrained reservoir architecture, characterized by uniform ring connectivity and binary input-to-reservoir weights with an aperiodic sign pattern. For linear reservoirs, given the reservoir size, the reservoir construction has only one degree of freedom -- the reservoir cycle weight. Such architectures are particularly amenable to hardware implementations without significant performance degradation in many practical tasks. 
In this study we endow these observations with solid theoretical foundations by proving that SCRs operating in real domain are universal approximators of time-invariant dynamic filters with fading memory.
Our results supplement recent research showing that SCRs in the complex domain can approximate, to arbitrary precision, any unrestricted linear reservoir with a non-linear readout.
We furthermore introduce a novel method to drastically reduce the number of SCR units, making such highly constrained architectures natural candidates for low-complexity hardware implementations. Our findings are supported by empirical studies on real-world time series datasets.
\end{abstract}

\maketitle

\section{Introduction}

Reservoir Computing (RC) is a subclass of Recurrent Neural Network defined by a fixed parametric state space representation (the reservoir) and a static trained readout map. This distinctive approach not only simplifies the training process by focusing adjustments solely on the static readout layer but also enhances computational efficiency.

The simplest recurrent neural network realization of RC \cite{Jaeger2001,Maass2002,Tino2001,Lukoservicius2009} are known as Echo State Networks (ESN) \cite{Jaeger2001,Jaeger2002,jaeger2002a,Jaeger2004}. 
The representation capacity of ESNs have been demonstrated in a series of papers, showing \emph{existentially} that ESNs can approximate any time-invariant dynamic filters with fading memory in a variety of settings  \cite{grigoryeva2018echo,Grigoryeva2018, Grigoryeva2018, gonon2019reservoir}. 
 
Simple Cycle Reservoirs (SCRs) \cite{rodan2010minimum} are RC models characterized by a highly restricted architecture: a uniform weight ring connectivity among reservoir neurons and binary uniformly weighted input-to-reservoir connections. This simplicity is particularly advantageous for hardware implementations, reducing implementation and computational costs, as well as enhancing real-time processing capabilities without degrading performance. 
However, while cyclic reservoir topology with a single connection weight has been adopted in a variety of hardware implementations  \cite{Appeltant2011InformationPU, NTT_cyclic_RC,bienstman2017,Abe2024May}, its theoretical foundations have been missing.   
To rectify this situation we rigorously prove
that SCRs operating in real domain are universal approximators of time-invariant dynamic filters with fading memory. Our results supplement recent research \cite{li2023simple} showing that SCRs in the complex domain can approximate,
to arbitrary precision, any unrestricted linear reservoir with a non-linear readout, opening the door to a broad range of practical scenarios involving real-valued computations. 
Furthermore, based on our constructive proof arguments, we formulate a novel method to drastically reduce the number of SCR units, making such highly constrained architectures natural candidates for low-complexity hardware implementations.

We emphasize that proving that SCR architecture retains universality when moving from the complex to the real domain is far from straightforward. Indeed, as shown in  \cite{li2023simple}, attempts to even partially restrict SCR in the complex domain to the real one result in more complex multi-reservoir structures.\footnote{ 
As we will show, to retain advantages of the single simple SCR structure in the real domain {\em and} retain universality one needs to consider   
orthogonal similarity throughout the approximation pipeline, as well as 
completion of the set of roots of unity in the canonical form of orthogonal matrices for
cyclic dilation of orthogonally dilated coupling matrices.}

We conclude the paper with numerical experiments that validate the structural approximation properties demonstrated in our theoretical analysis. In particular, we begin with an arbitrary linear reservoir system with linear readout and demonstrate, on real-life datasets, that the approximation SCR counterparts will gradually approach the original system as the number of reservoir neurons increase, reinforcing the theoretical insights provided.

\section{The setup}

We first introduce the basic building blocks needed for the developments in this study.

\begin{definition}\label{def.lrc} A \textbf{linear reservoir system over $\mathbb{R}$} is formally defined as the triplet $R:= (W,V,h)$ where the \textbf{dynamic coupling} $W$ is an $n\times n$ real-valued matrix, the \textbf{input-to-state coupling} $V$ is an $n\times m$ real-valued matrix, and the state-to-output mapping (\textbf{readout}) $h:\mathbb{R}^n \to \mathbb{R}^d$ is a (trainable) continuous function. 

The corresponding linear dynamical system is given by:
\begin{equation} \label{eq.system}
   \begin{cases} \Bx_t &= W \Bx_{t - 1} + V \u_t \\
    \By_t &= h(\Bx_t)
    \end{cases}
\end{equation}
where $\{\u_t\}_{t\in\mathbb{Z}_-} \subset \mathbb{R}^m$, $\{\Bx_t\}_{t\in\mathbb{Z}_-} \subset \mathbb{R}^n$, and $\{\By_t\}_{t\in\mathbb{Z}_-} \subset \mathbb{R}^d$ are the external inputs, states and outputs, respectively.
We abbreviate the dimensions of $R$ by $(n,m,d)$.
\end{definition}

We make the same assumption as in \cite[Definition 1]{li2023simple}, that
\begin{enumerate}
    \item The input stream $\{\u_t\}$ is uniformly bounded by a constant $M$.
    \item The dynamic coupling matrix $W$ has its operator norm $\norm{W}<1$. 
\end{enumerate}
The only difference in this present work is the requirement that all the matrices and vectors are over $\mathbb{R}$. 
Under the assumptions, for each left infinite time series
$u = \{\u_t\}_{t\in\mathbb{Z}_-}$, the system~\eqref{eq.system} has a unique solution given 
by
\begin{align*}
    {\Bx_t}(u) &=\sum_{n\geq 0} W^n V \u_{t-n}, \\
    {\By_t}(u) &= h({{\Bx_t}(u)}). 
\end{align*}
We refer to the solution simply as $\{(\Bx_t,\By_t)\}_t$.

\begin{definition} For two reservoir systems $R=(W,V,h)$ (with dimensions $(n,m,d)$) and $R'=(W', V', h')$ (with dimensions $(n',m,d)$): 
\begin{enumerate}
    \item We say the two systems are \textbf{equivalent} if for any input stream $u=\{\u_t\}_{t\in\mathbb{Z}_-}$, the solutions $\{(\Bx_t,\By_t)\}_t$ and $\{(\Bx'_t,\By'_t)\}_t$ for systems $R$ and $R'$
     satisfy ${\By}_t={\By}_t'$ for all $t$.
    
    \item For $\epsilon>0$, we say the \textbf{two systems are $\epsilon$-close} if for any input stream $u=\{\u_t\}_{t\in\mathbb{Z}_-}$, the solutions of the two systems (under the notation above) satisfy  $\norm{{\By}_t-{\By}_t'}_2<\epsilon$ for all $t$.
\end{enumerate}
\end{definition}

We now define the main object of interest in this paper. We begin by the following definition.

\begin{definition} Let $P=[p_{ij}]$ be an $n\times n$ matrix. We say $P$ is a \textbf{permutation matrix} if there exists a permutation $\sigma$ in the symmetric group $S_n$ such that:
\[p_{ij}=\begin{cases} 1, &\text{ if }\sigma(i)=j, \\0, &\text{ if otherwise.}\end{cases}\]
We say a permutation matrix $P$ is a \textbf{full-cycle permutation}\footnote{Also called left circular shift or cyclic permutation in the literature} if its corresponding permutation $\sigma\in S_n$ is a cycle permutation of length $n$.  Finally, a matrix $W=cP$ is called a \textbf{contractive full-cycle permutation} if $c\in(0,1)$ and $P$ is a full-cycle permutation. 
\end{definition}

Simple cycle reservoir systems originate from the minimum complexity reservoir systems  {introduced} in \cite{rodan2010minimum}:

\begin{definition}
\label{def.rc}
    A linear reservoir system $R = \left(W,V,h\right)$ with dimensions $(n,m,d)$ is called a \textbf{Simple Cycle Reservoir (SCR) over $\mathbb{R}$}
    \footnote
     {We note that the assumption on the aperiodicity of the sign pattern in $V$ is not required for this study}
    if 
    \begin{enumerate}
        \item $W$ is a contractive full-cycle permutation, and
        \item $V \in \mathbb{M}_{n \times m}\left(\left\{-1,1\right\}\right)$. 
    \end{enumerate}
\end{definition}

Our goal is to show that every linear reservoir system is $\epsilon$-close to a SCR over $\mathbb{R}$. This is done in a way that does not increase the complexity of the  {readout} function $h$. To be precise:

\begin{definition}
We say that \textbf{a function $g$ is $h$ with linearly transformed domain} if $g(\x) = h(A\x)$ for some matrix $A$. 
\end{definition}

\section{Universality of Orthogonal Dynamical Coupling}
\label{sec.dilation}

In \cite{Halmos1950}, Halmos raised the question of what kind of operators can be embedded as a corner of a normal operator. 
He observed that one can embed a contractive operator $W$ inside a unitary operator:
\[\tilde U=\begin{bmatrix}
W & D_{W^\top} \\
D_W & -W^\top
\end{bmatrix}\] 
where $D_W=(I-W^\top W)^{1/2}$ and $D_{W^\top}=(I-WW^\top)^{1/2}$. 
This motivated the rich study of the dilation theory of linear operators. One may refer to \cite{PaulsenBook} for more comprehensive background on dilation theory. In the recent study \cite{li2023simple}, the authors used a dilation theory technique to obtain an $\epsilon$-close reservoir system with a unitary matrix as the dynamic coupling. A key idea is the theorem of Egerv\'{a}ry, which states that for any $n\times n$ matrix $W$ with $\|W\|\leq 1$ and $N>1$, there exists a $(N+1)n\times (N+1)n$ unitary matrix $U$ such that the upper left $n\times n$ corner of $U^k$ is $W^k$ for all $1\leq k\leq N$. In fact, this matrix $U$ can be constructed explicitly as:
\[U=\begin{bmatrix}
W   & 0  & 0 & & \hdots      & 0   & D_{W^\top} \\
D_W & 0  & 0 & & \hdots      & 0   & -W^\top \\
 0   & I &  0     & \hdots   & & 0 & 0 \\
 \vdots   &   & \ddots&    &   &  \vdots & \vdots\\
 0   & \hdots  &  &  &   & I  & 0 
\end{bmatrix}.\]
We first notice that when $W$ is a matrix over $\mathbb{R}$, this dilation matrix $U$ is over $\mathbb{R}$ as well. Therefore, $U$ is an orthogonal matrix. This technique allows us to obtain an $\epsilon$-close reservoir system with an orthogonal dynamic coupling matrix.

\begin{theorem}\label{thm.dilation} Let $R=(W,V,h)$ be a reservoir system defined by contraction $W$ with $\norm{W}=:\lambda\in (0,1)$. Given $\epsilon>0$, there exists a reservoir system $R'=(W', V', h')$ that is $\epsilon$-close to $R$, with dynamic coupling $W'=\lambda U$, where $U$ is orthogonal. Moreover, $h'$ is $h$ with linearly transformed domain.
\end{theorem}

\begin{proof}
{The proof follows that of an analogous statement in the complex domain in {\cite[Theorem 11]{li2023simple}}. The arguments follow through by replacing unitary matrices by orthogonal matrices and conjugate transpose by regular transpose.}
For completeness we present the proof in Appendix \ref{pf:thm.dilation}.
\end{proof}

\section{Universality of Cyclic Permutation Dynamical Coupling}
Having established universality of orthogonal dynamic coupling in the reservoir domain, we now continue to show that for any reservoir system with {orthogonal} state coupling, we can construct an equivalent reservoir system with cyclic coupling of state units weighted by a single common connection weight value. 
In broad terms we will employ the strategy of \cite{li2023simple}, but here we need to pay {special} attention to maintain all the matrices in the real domain.

We begin by invoking \cite[Proposition 12]{li2023simple}, which stated that matrix similarity of dynamical coupling implies reservoir equivalence. It therefore remains to be shown that for any given \emph{orthogonal} state coupling we can always find a full-cycle permutation that is close to it to arbitrary precision. Specifically, when given an orthogonal matrix, the goal is to perturb it to another orthogonal matrix that is \emph{orthogonally} equivalent to a permutation matrix. 
Here we \textbf{cannot} adopt the strategy in \cite[Section 5]{li2023simple} because it would inevitably involve a unitary matrix over $\mathbb{C}$ during the diagonalization process. Instead, we convert an orthogonal matrix to its canonical form via a (real) orthogonal matrix.

The core idea of our approach is schematically illustrated in Figure~\ref{fig:test}. Given a reservoir system with dynamic coupling $W$, we first find its equivalent with orthogonal coupling $U$. Rotational angles in the canonical form of $U$ are shown as red dots (a). Roots of unity corresponding to a sufficiently large cyclic dilation coupling can approximate the rotational angles to arbitrary precision $\epsilon$ (b). 

\begin{figure*}[ht!]
\centering
\begin{subfigure}{.4\textwidth}
  \centering
  \includegraphics[width=0.6\linewidth]{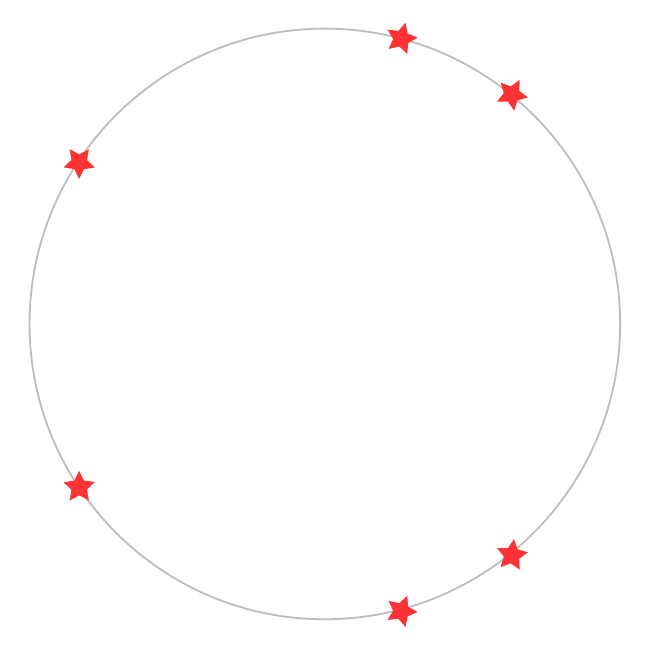}
  \caption{}
  \label{fig:sub1}
\end{subfigure}%
\begin{subfigure}{.6\textwidth}
  \centering
  \includegraphics[width=0.8\linewidth]{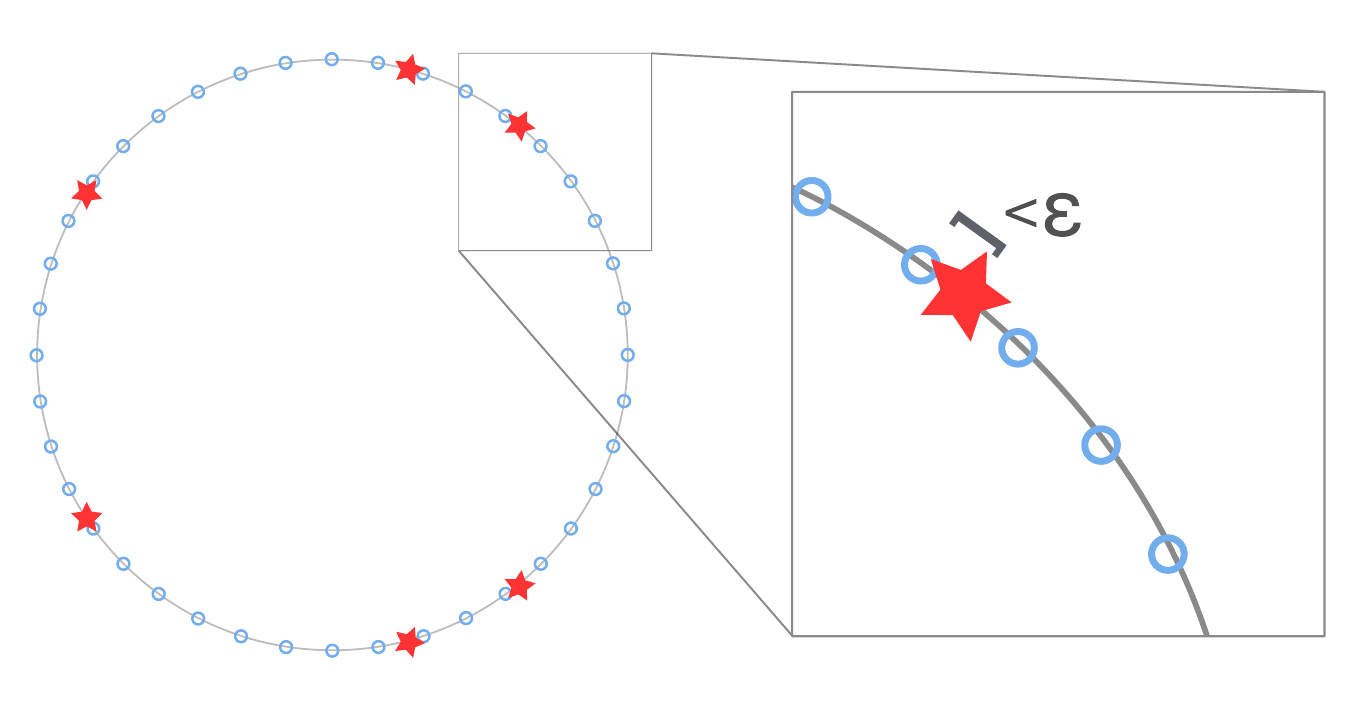}
  \caption{}
  \label{fig:sub2}
\end{subfigure}
\caption{
Schematic illustration of the core idea enabling us to prove universality of SCRs in the real domain. Given a reservoir system with dynamic coupling $W$ and its equivalent with orthogonal coupling $U$,
rotational angles in the canonical form of $U$ are shown as red dots (a). Roots of unity corresponding to cyclic dilation approximate the rotational angles to a prescribed precision $\epsilon$ (b).}

\label{fig:test}
\end{figure*}

We begin by recalling some elementary results of orthogonal and permutation matrices. For each $\theta\in [0, 2\pi)$, define the following rotation matrix:
\[R_\theta = \begin{bmatrix} \cos(\theta) & -\sin(\theta) \\ \sin(\theta) & \cos(\theta)\end{bmatrix}.\]
The eigenvalues of $R_\theta$ are precisely $e^{\pm i\theta}$. Moreover, notice that:
\begin{align*}
\begin{bmatrix} 0 & 1 \\ 1 & 0\end{bmatrix} R_{- \theta} \begin{bmatrix} 0 & 1 \\ 1 & 0\end{bmatrix}^{-1} &= \begin{bmatrix} 0 & 1 \\ 1 & 0\end{bmatrix} \begin{bmatrix} \cos(\theta) & \sin( \theta) \\ -\sin(\theta) & \cos(\theta)\end{bmatrix} \begin{bmatrix} 0 & 1 \\ 1 & 0\end{bmatrix} \\
& = \begin{bmatrix} \cos(\theta) & -\sin(\theta) \\ \sin(\theta) & \cos(\theta)\end{bmatrix} = R_\theta.
\end{align*}
Therefore, $R_\theta$ and $R_{-\theta}$ are orthogonally similar. 
Employing the real version of Schur decomposition, for any orthogonal matrix $C \in O(n)$, there exists an orthogonal matrix $S$ such that the product $S^\top C S$ has the following form:

\begin{align*}
S^\top C S &= \begin{bmatrix} R_{\theta_1} & & & & &\\
& \ddots &&&&\\
&& R_{\theta_k} &&& \\
&&& \pm 1 && \\
&&&  & \ddots&\\
&&&&& \pm 1
\end{bmatrix} \\
&=
\begin{bmatrix} R_{\theta_1} & & &\\
& \ddots &&\\
&& R_{\theta_k} & \\
&&& \Upsilon
\end{bmatrix},
\end{align*}
where $\theta_i\in (0,\pi)$, and $\Upsilon := \operatorname{diag}\{a_1, a_2,...,a_q\}$, $a_i \in \{-1,+1\}$,
$i=1,2,...,q$, is a diagonal matrix with $q$ entries of $\pm 1$'s.
For simplicity, 
\textbf{we will assume for the rest of the paper an even dimension $n$}, which inherently implies that $q$ is also even. The case when $n$ is odd is analogous and follows from similar arguments.

Note that without loss of generality, $\Upsilon$ can always take the form where $+1$'s (if any) preceded $-1$'s (if any). This can be achieved by permuting rows of $\Upsilon$ which is an orthogonality preserving operation  {and invoking \cite[Proposition 12]{li2023simple}}. This can be further simplified by observing:
\[R_0 = \begin{bmatrix} 1 & 0 \\ 0 & 1\end{bmatrix}, \ R_\pi = \begin{bmatrix} -1 & 0 \\ 0 & -1\end{bmatrix}.\]
 {Hence, pairs} of $+1$'s (and $-1$'s) can therefore be grouped as blocks of $R_0$ (and $R_\pi$). Therefore, without loss of generality, $S^\top C S$ is a block diagonal matrix consisting of $\{R_{\theta_1},\dots, R_{\theta_m}\}$,
$\theta_i\in[0,\pi]$, 
$i=1,2,...,m$, and at most one block of the form
$
\begin{bmatrix} 1 & 0 \\ 0 & -1\end{bmatrix}.$
 {In the literature} this is known as the \textbf{canonical form of the orthogonal matrix} $C$.

Given the inherent orthogonality of permutation matrices, their corresponding canonical forms can be computed. Given an integer $\ell \ge 1$, the complete set of $\ell$-th roots of unity is a collection of {uniformly positioned points} along the complex circle, denoted by  $\{e^{i\frac{2j\pi}{\ell}}: 0\leq j\leq \ell-1\}$. 

It is well-known from elementary matrix analysis that the eigenvalues of a full-cycle permutation matrix form a complete set of roots of unity. 

Therefore, given an $\ell\times \ell$ full-cycle permutation $P$ ($\ell$ even), we can find  {an} orthogonal matrix $Q$ such that $Q^\top P Q$ is a block diagonal matrix of $\{1, -1, R_{\frac{2\pi j}{\ell}}: 1\leq j < \frac{\ell}{2}\}$. Here, note that for each $1\leq j<\frac{\ell}{2}$, $R_{\frac{2\pi j}{\ell}}$ has two  {conjugate eigenvalues $e^{i\frac{2\pi j}{\ell}}$ and $e^{-i\frac{2\pi j}{\ell}}$.}
Hence, an $\ell \times \ell$ orthogonal matrix $X$ is orthogonally equivalent to a full-cycle permutation if and only if its canonical form consists of:
\begin{enumerate}
    \item A complete set of rotation matrices $\{R_{\frac{2\pi j}{\ell}}: 1\leq j < \frac{\ell}{2}\}$, and
    \item Two additional diagonal entries of $1$ and $-1$.
\end{enumerate}

\begin{theorem}\label{thm.perturb.unitary} Let $U$ be an $n\times n$ orthogonal matrix and $\delta>0$ be an arbitrarily small positive number. There exists $n_1 \ge n$, an $n_1\times n_1$ orthogonal matrix $S$, an $n_1\times n_1$ full-cycle permutation $P$ and an $(n_1-n)\times (n_1-n)$ orthogonal matrix $D$, such that 
\[
\norm{S^\top P S-\begin{bmatrix} U & 0 \\ 0 & D\end{bmatrix}}<\delta .
\]
\end{theorem}

\begin{proof} 

We only prove the case when the canonical form of $U$ is in a block diagonal form consisting of $R_{\theta_1}, R_{\theta_2}, \dots, R_{\theta_k}$. The case when the canonical form contains additional entries of $\pm 1$ is analogous. Let $S_1$ be the orthogonal matrix such that $S_1^\top U S_1$ is in the canonical form. For fixed $\delta > 0$, pick an integer $\ell_0 > 0$ such that:
    $|1-e^{\frac{\pi i}{\ell_0}}|<\delta$.

For each $j = 1,\ldots k$, the interval $I_j := \left(\frac{\theta_j \ell_0 (k+1)}{\pi} - (k+1),\  \frac{\theta_j \ell_0 (k+1)}{\pi} + (k+1)\right)$ has length $2(k+1)$, and therefore contain $2(k+1)$ distinctive integers. Moreover, since $\theta_j\in [0,\pi]$, the interval $I_j$ contains at least $k$ distinct integers strictly within $(0, \ell_0 \cdot (k+1))$. From each interval $I_1,\ldots, I_k$ we can thus choose distinct integers $a_1, \cdots, a_k$, with $a_j \in I_j$ and $0< a_j < \ell_0 \cdot \left(k+1\right)$ such that for each $j = 1,\ldots,k$:
\[ 0 <\left|\frac{\theta_j \ell_0 (k+1)}{\pi} - a_j\right| < (k+1),\]
or equivalently,
\[ 0 <\left|\theta_j - \frac{a_j}{\ell_0 (k+1)} \cdot \pi\right| < \frac{\pi}{\ell_0}.\]

This implies: 
 {
\begin{align*}
    & \left|e^{i\theta_j} - e^{\pi i \frac{a_j}{\ell_0 \cdot (k+1)}}\right| = \left|1 - e^{i \left(-\theta_j + \frac{a_j}{\ell_0 \cdot (k+1)} \cdot \pi\right)}\right| \\
    & \quad = \left|1 - e^{i \left|\theta_j - \frac{a_j}{\ell_0 \cdot (k+1)} \cdot \pi\right|}\right|  \leq  \left|1 - e^{i \frac{\pi}{\ell_0}}\right| < \delta.
\end{align*}
}

Now for each $j = 1,\ldots,k$, set $\beta_j := \frac{\pi a_j}{\ell_0 (k+1)} \in (0,\pi)$. 
Let $A_0$ be the block diagonal matrix consisting of $\{R_{\beta_1}, \dots, R_{\beta_k}\}$. Since the eigenvalues of $R_{\beta_j}$ is within $\delta$ to $R_{\theta_j}$, we obtain:
\begin{align*}
    \norm{A_0 - S_1^\top U S_1} &= \norm{\begin{bmatrix} R_{\beta_1}-R_{\theta_1} & & \\
& \ddots &\\
&& R_{\beta_k}-R_{\theta_k} 
\end{bmatrix}} \\
&=\max\{\norm{R_{\beta_j} - R_{\theta_j}}\} < \delta,
\end{align*}
where $\norm{\cdot}$ denotes the operator norm.

Let $n_1 := 2\ell_0(k+1)$. We have each $\beta_j=\frac{a_j}{n_1} \cdot (2\pi)$, and therefore the rotations $R_{\beta_1}, \dots, R_{\beta_k}$ are all rotations of distinct $n_1$-roots of unity. Whilst this is not a \textit{complete} set of rotations of the roots of unity, we can complete the set of rotations by filling in the missing ones. In particular the missing set of rotations is given explicitly by: 
 {
\begin{align*}
    \mathcal{R}_1 := \left\{R_{\frac{2 \pi a}{n_1}}: 
    a \in \mathbb{Z}, 
    0<a<\frac{n_1}{2}, a\neq a_j\right\}
\end{align*}
}
Let $D$ denote the $\left(n_1 - n\right) \times \left(n_1 - n\right)$ block diagonal matrix consisting of:
\begin{enumerate}
    \item All the missing blocks of rotations described in $\mathcal{R}_1$, and
    \item { Two additional diagonal entries of $1$ and $-1$. 
    }
\end{enumerate} 
By construction $D$ is orthogonal since it contains block diagonal matrices of $\pm1$ and $R_\theta$. Consider the $n_1 \times n_1$ matrix $A :=(S_1 A_0 S_1^\top)\oplus D$. Then $A$ is orthogonal by construction and the canonical form of $A$ consists of:
\begin{enumerate}
    \item A complete set of rotations $R_{\frac{\pi a}{n_1}}, a \in \mathbb{Z}, 0<a<\frac{n_1}{2}$, and 
    \item An additional diagonal entry of $1$ when $n_1$ is odd and two additional diagonal entries of $1$ and $-1$ when $n_1$ is even.
\end{enumerate} 

This is precisely the canonical form of a $n_1 \times n_1$ full-cycle permutation. Let $P$ be a full-cycle permutation of dimension $n_1$, then there exists an orthogonal matrix $S$ such that $A=S^\top P S$. We have,
\begin{align*}\norm{S^\top P S - \begin{bmatrix} U & 0 \\ 0 & D\end{bmatrix}} &= \norm{\begin{bmatrix} S_1 A_0 S_1^\top - U & 0 \\ 0 & 0\end{bmatrix}}\\
&=\norm{A_0 - S_1^\top U S_1}<\delta, 
\end{align*}
as desired.
\end{proof}

\begin{remark} 
\label{rmk:graphmatching}
In practice, the dimension $n_1$ is usually much smaller than the theoretical upper bound of $2\ell_0(k+1)$. Here, the integer $\ell_0$ is chosen to satisfy $|1-e^{\frac{\pi i}{\ell_0}}|<\delta$, which equivalently means:
\begin{align*}
    \frac{\pi}{\ell_0} < \operatorname{arccos}\left(1-\frac{\delta^2}{2}\right). 
\end{align*}
In practice, a much lower dimension can be achieved.
Given a set of angles $\left\{\theta_i\right\}$ from an $n\times n$ orthogonal matrix $U$. For a fixed $n' > n$, we can use maximum matching program in a bipartite graph to check whether: for each $\theta_i$ there exists a distinct $k_i$ such that the root-of-unity $\frac{2 k_i \pi}{n'}$ approximates $\theta_i$. For fixed $n'$, define a bipartite graph $G$ with vertex set $A\cup B$ with $A=\{\theta_i\}$ and $B=\{\frac{2 a\pi}{n'}: 0<a<\frac{n'}{2}\}$. An edge $e$ joins $\theta_i\in A$ with $\frac{2 a\pi}{n'}\in B$ if $\left|e^{\theta_i i} - e^{\frac{2 a\pi i}{n'}}\right|<\delta $. One can easily see that we can find distinct $k_i$ to approximate $\theta_i$ by $\frac{2 k_i \pi}{n'}$ if and only if there exists a matching for this bipartite graph with exactly $|A|$ edges. We let $n_C$ denote the smallest $n'$ such that a desired maximum matching is achieved. We shall demonstrate later, in the numerical experiment section, that $n_C$ is significantly lower than the theoretical upper bound given by Theorem~\ref{thm.perturb.unitary}.
\end{remark}

\begin{theorem}\label{thm.to.permutation} Let $U$ be an $n\times n$ orthogonal matrix and $W = \lambda U$ with $\lambda \in \left(0,1\right)$. Let $R=(W,V,h)$ be a reservoir system with state coupling $W$. For any $\epsilon>0$, there exists a reservoir system $R_c=(W_c, V_c, h_c)$ that is $\epsilon$-close to $R$ such that:
\begin{enumerate}
\item $W_c$ is a contractive full-cycle permutation with $\|W_c\|=\|W\| = \lambda \in \left(0,1\right)$, and
\item $h_c$ is $h$ with linearly transformed domain. 
\end{enumerate}
\end{theorem}

\begin{proof}
The proof follows that of an analogous statement in the complex domain in \cite[Theorem 14]{li2023simple}. The arguments follow through  by replacing unitary matrices by orthogonal matrices and conjugate transpose by regular transpose.

For completeness we present the proof in Appendix \ref{pf:thm.topermutation}
\end{proof}

\section{From Cyclic Permutation to SCR}

We have now ready to prove the main result (Theorem~\ref{thm.main.scr}). So far, we have proved that any linear reservoir system $R=(W, V, h)$ is $\epsilon$-close to another reservoir system $R'=(W', V', h')$ where $W'$ is a permutation or a full-cycle permutation. It remains to show that one can make the entries in the  {input-to-state coupling} matrix $V$ to be all $\pm 1$. 

We first recall the following useful Lemmas.

\begin{lemma}[{\cite[Lemma 16]{li2023simple}}]\label{lm.fullcycleW} Let $n,k$ be two natural numbers such that $\gcd(n,k)=1$. Let $P$ be an $n\times n$ full-cycle permutation. Consider the $nk\times nk$ matrix:
\[P_1=\begin{bmatrix} 
0   & 0  & 0 & & \hdots      & 0   & P \\
P & 0  & 0 & & \hdots      & 0   & 0 \\
 0   & P &  0     & \hdots   & & 0 & 0 \\
 \vdots   &   & \ddots&    &   &  \vdots & \vdots\\
 0   & \hdots  &  &  &   & P  & 0 
\end{bmatrix}.\]
Then $P_1$ is a full-cycle permutation. 
\end{lemma}

\begin{lemma}[{\cite[Lemma 17]{li2023simple}}]\label{lm.simpleV} For any $n\times m$ real matrix $V$ and $\delta>0$, there exists $k$ matrices $\{F_1, \cdots, F_k\} \subset \mathbb{M}_{n \times m}\left(\left\{-1,1\right\}\right)$ and a constant integer $N>0$ such that:
\[\norm{V-\frac{1}{N} \sum_{j=1}^k F_j} < \delta\]
Moreover, $k$ can be chosen such that $\gcd(k,n)=1$. 
\end{lemma}

We now obtain our main theorem on the universality of SCR over $\mathbb{R}$. In comparison to \cite[Theorem 20]{li2023simple}, the coupling matrix $V$ in a SCR over $\mathbb{R}$ contains only $\pm 1$ instead of $\{\pm 1, \pm i\}$.

\begin{theorem}\label{thm.main.scr} For any reservoir system $R=(W,V,h)$ of dimensions $(n,m,d)$ and any $\epsilon>0$, there exists a SCR $R'=(W',V', h')$ of dimension $(n',m,d)$ that is $\epsilon$-close to $R$. 
Moreover, $\norm{W}=\norm{W'}$ and  $h'$ is $h$ with linearly transformed domain. 

\end{theorem}

\begin{proof} {
One may refer to the proof of \cite[Theorem 20]{li2023simple}. Crucially, because the dynamic coupling matrix $V$ in the intermediate steps are all over $\mathbb{R}$ instead of $\mathbb{C}$, the resulting matrix $V'$ only have $\pm 1$.     }
\end{proof}

\cite{Grigoryeva2018}(Corollary 11) shows that linear reservoir systems with polynomial readouts are universal. They can approximate to arbitrary precision time-invariant fading memory filters. 
This result, together with Theorem \ref{thm.main.scr},  establish universality of SCRs in the real domain. Indeed, given a time-invariant fading memory filter, one can find an approximating linear reservoir system with polynomial readout $h$ approximating the filter to the desired precision.
By Theorem \ref{thm.main.scr}, we can in turn \emph{constructively} approximate this reservoir system with a SCR, again to arbitrary precision.
Moreover, the SCR readout is a polynomial of the same degree as $h$, since it is $h$ with linearly transformed domain.

\section{Numerical Experiments}
\label{sec.experiment}
We conclude the paper with numerical experiments  illustrating our contributions.

For reproducibility of the experiments, all experiments are CPU-based and are performed on Apple M3 Max with 128GB of RAM. The source code and data of the numerical analysis is openly available at \texttt{https://github.com/Lampertos/RSCR}.

\subsection{Dilation of Linear Reservoirs on Time Series Forecasting}
\label{sec.exp.dilation}
This section illustrates the structural approximation properties of linear reservoir systems when dilating the reservoir coupling matrix, when applied to univariate time series forecasting. 
{The readout function will be assumed to be linear throughout the numerical analysis in this section, as the primary objective of this paper is to examine the structural approximation properties of the state-space representation of linear reservoirs as determined by the coupling matrix.}

Initially, a linear reservoir system featuring a randomly generated coupling matrix $W$ is constructed. We then approximate this system by linear reservoir systems under two types of dilated coupling matrices:
(1) $U$ – Orthogonal dilation of $W$ (Theorem~\ref{thm.dilation}) , and
(2) $C$ – Cyclic dilation of $U$ (Theorem~\ref{thm.to.permutation}).

The results demonstrate that the prediction loss approximation error diminishes progressively as the dimension of dilation increases. 
{For demonstration purposes, we keep dimensionality of the original reservoir system to be approximated by SCR low ($n=5$). The elements of $W$ are independently sampled from the uniform distribution $U(0,1)$. The elements of input-to-state coupling $V$ is generated by scaling the binary expansion of the digits of $\pi$ by $0.05$ \cite{rodan2010minimum}.}
Univariate forecasting performance  of the initial and the approximating reservoir systems are compared on two popular datasets used in recent time series forecasting studies (e.g. \cite{zhou2021informer}
(adopting their training/validation/test data splits)):

\paragraph{ETT}
The Electricity Transformer Temperature dataset\footnote{\texttt{https://github.com/zhouhaoyi/ETDataset}}consists of measurements of oil temperature
and six external power-load features from transformers in two regions of China. The data was recorded for two years, and measurements are provided either hourly (indicated by 'h') or every $15$ minutes (indicated by 'm'). In this paper we used \texttt{oil temperature} of the \texttt{ETTm2} dataset for univariate prediction with train/validation/test split being $12$/$4$/$4$ months.

\paragraph{ECL}
The Electricity Consuming Load%
\footnote{\texttt{https://archive.ics.uci.edu/dataset/321/ \\ electricityloaddiagrams20112014}} consists of
hourly measurements of electricity consumption in kWh for 321 Portuguese clients during two years.
In this paper we used client \texttt{MT 320} for univariate prediction. The train/validation/test split is 15/3/4 months.

The readout $h$ of the original reservoir system is trained using ridge regression with a ridge coefficient of $10^{-9}$. Note that modified versions of the input-to-state map $V$ and the readout map $h$ will be employed in all subsequent dilated systems. Specifically, the readout map will not be subject to further training in these systems.
In all simulations, we maintain a spectral radius $\lambda = 0.9$ and prediction horizon is set to be $300$.\footnote{These two parameters are primarily chosen not for accuracy of the forecasting capacities but to demonstrate the structural approximation properties proven in this paper.}

The initial system $R=(W,V,h)$ is dilated over a set of pre-defined dilation dimensions $\mathcal{D} := \{ 2,  6, 10, 15, 19, 24, 28, 33, 37, 42\}$\footnote{Recall that the corresponding orthogonal dilation will have dimensions $\left(N+1\right)\cdot 5 \times \left(N+1\right) \cdot 5$ for each $N \in \mathcal{D}$.}. For each $N \in \mathcal{D}$, by Theorem~\ref{thm.dilation}, we construct a linear reservoir system $R_U$ with an orthogonal dynamic coupling $W_U$ of dimension $n_U =(N+1)n$. Then by Theorem~\ref{thm.to.permutation}, we dilate $R_U$ into an $\epsilon$-close linear reservoir system $R_C$ with contractive cyclic-permutation dynamic coupling. 

For the orthogonal dilation, the linear reservoir system $R_U := \left(W_U, V_U, h_U\right)$  is defined by:
\begin{align*}
    W_U &:= \lambda \cdot U, \text{ where } \\
    U &:= \begin{bmatrix}
        W   &   &       &    & D_{W^\top} \\
        D_W &   &       &    & -W^\top \\
            & I &       &    & \\
            &   & \ddots&    &   \\
            &   &       & I  & 0 
        \end{bmatrix} \in \mathbb{M}_{5 \cdot\left(N+1\right)\times 5 \cdot \left(N+1\right)}(\mathbb{R}) \\
 V_U &:=\begin{bmatrix}
    V  \\
    0
    \end{bmatrix},\quad h_U(x) = h\left(P_n(x)\right),
\end{align*}
where $P_n:\mathbb{R}^{5\left(N+1\right)}\hookrightarrow \mathbb{R}^5$ denote the projection onto the first $n = 5$ coordinates.

Since $U$ is orthogonal, it's canonical form $T_U$ can be obtained via the real version of Schur's decomposition:
\begin{align*}
    U = J_U T_U J_U^\top,
\end{align*}
where we note that both $U$ and $T_U$ have unit spectral radius.

By Remark~\ref{rmk:graphmatching}, given $\epsilon > 0$ and the canonical form $T_U$ of $U$, the maximum matching program in bipartite graphs allows us to find the canonical form of the $n_C\times n_C$-dimensional root-completed-matrix $A$ (along with the corresponding dimension $n_C$), given by $T := A_0 \oplus D$, described in the proof of Theorem~\ref{thm.perturb.unitary}. By construction $T$ is $\epsilon$-close to $T_U \oplus D$ in terms of operator norm.

Let $C$ be the $n_C\times n_C$-dimensional full cycle permutation matrix of unit spectral radius. Since $C$ is orthogonal, we can once again apply Schur's decomposition to obtain it's canonical form:
\begin{align*}
    C = J_C T_C J_C^\top.
\end{align*}

By construction, $T$ contains rotational matrices with angles at the roots of unity that are rearrangements of the eigenvalues of cyclic permutation of the same size. Therefore there exists permutation matrix $\tilde{P}$ such that:
\begin{align*}
    \tilde{P}T\tilde{P}^\top = T_C.
\end{align*}

Therefore by the proof of Theorem~\ref{thm.to.permutation} (following that of \cite[Theorem 14]{li2023simple}), the linear reservoir system $R_C:= \left(W_C, V_C, h_C\right)$ for the cyclic dilation is defined by:
\begin{align}
\label{eq:P}
    W_C &:= \lambda \cdot C, \quad 
    V_C:=P \begin{bmatrix}
    V_U  \\
    0
    \end{bmatrix}, \nonumber \\
 h_C(\Bx) &= h\left(P_{n}(P^\top \Bx)\right), \quad P := J_C\tilde{P}\overline{J}_U^\top, 
\end{align}
where $C\in \mathbb{M}_{n_c \times n_c}(\mathbb{R})$ is a cyclic permutation of dimension $n_c > 5\cdot \left(N+1\right)$, and $\overline{J}_U$ is a $n_c - 5\cdot \left(N+1\right) \times n_c - 5\cdot \left(N+1\right)$ matrix with an $J_U$ on the upper left hand corner and zero everywhere else.

By the uniform continuity of the readout $h$, it suffices to evaluate the closeness 
of the state trajectories of the original and the approximating cyclic dilation systems, as they are driven by the same input time series \texttt{ETTm2} and \texttt{ECL}. The two state activation sequences are not directly comparable, but they become comparable if the states of the approximating system are transformed by the orthogonal matrix $P$ (Equation~\eqref{eq:P}) and projected into the first $n$ coordinates. Figure~\ref{fig:norm_diff} shows the mean square differences between the states of the two systems as a function of the dilation dimension $N$. As expected, MSE between the states decays to zero exponentially as the dilation dimension increases.

\begin{figure}[ht!]
\centering
\begin{subfigure}{.5\columnwidth}
  \centering
  \includegraphics[width=1.1\linewidth]{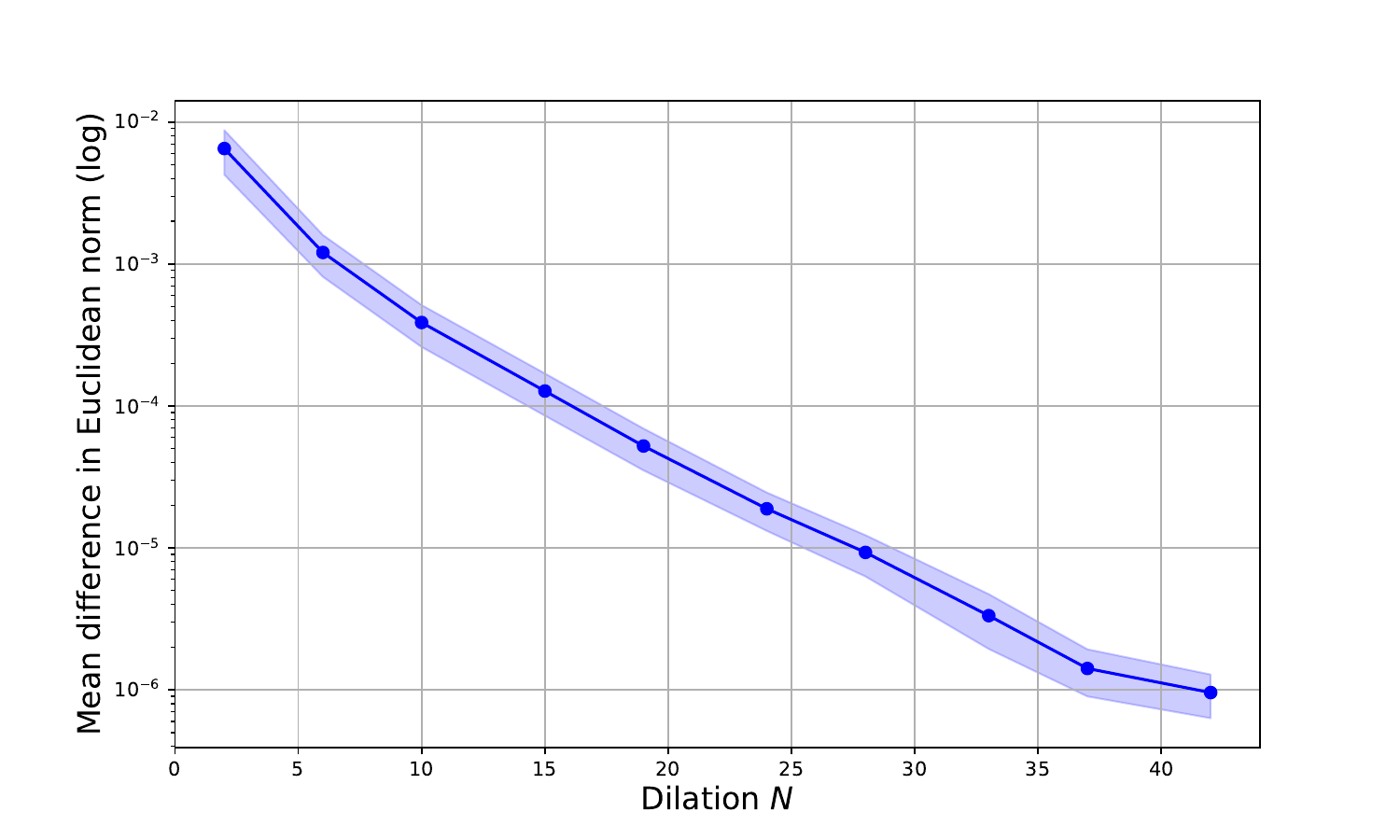}
  \caption{ECL}
\end{subfigure}%
\begin{subfigure}{.5\columnwidth}
  \centering
  \includegraphics[width=1.1\linewidth]{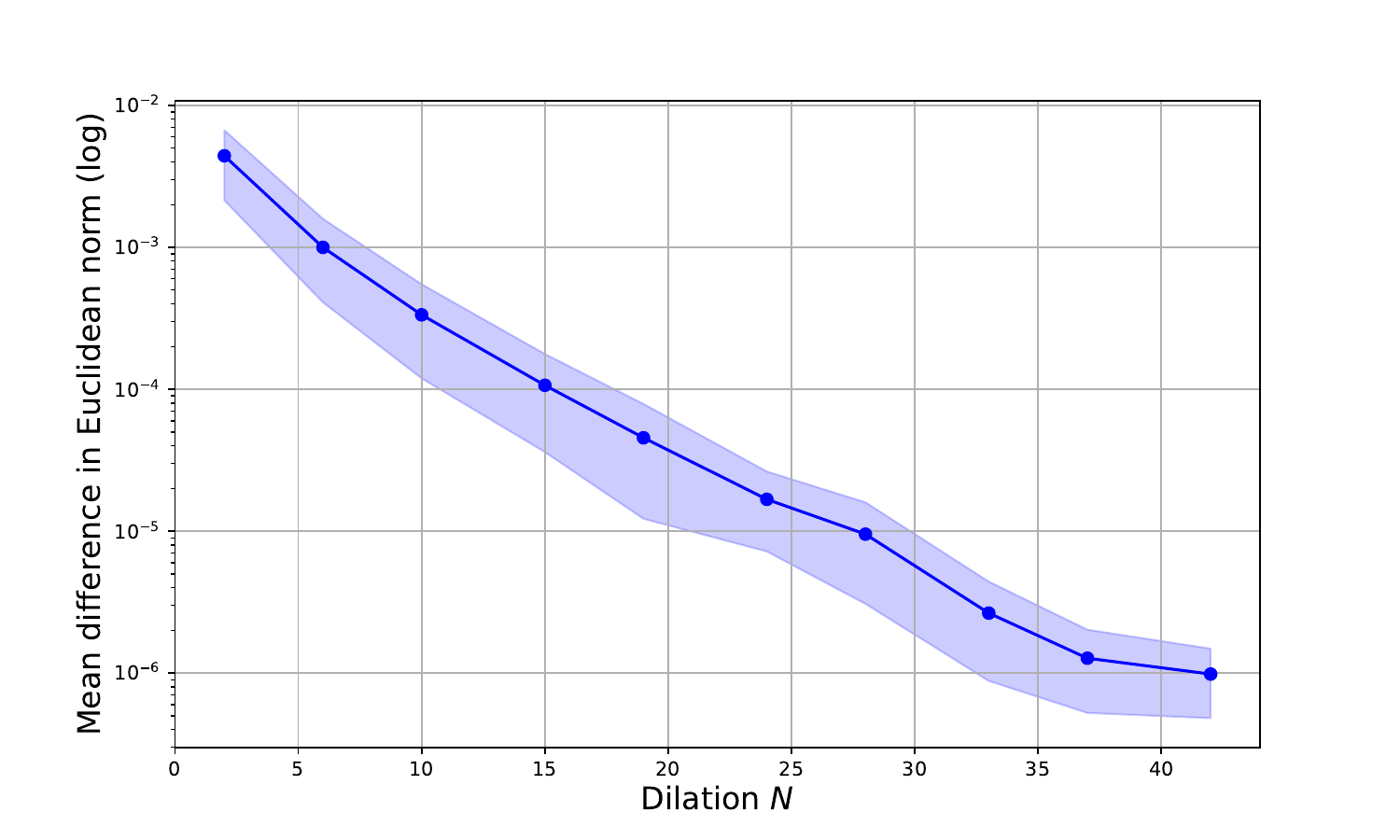}
  \caption{\texttt{ETTm2}}
\end{subfigure}
\caption{Mean and $95\%$ confidence intervals of the mean square differences of the states 
 of the original reservoir and the approximating cyclic dilation systems over 15 randomized generations of the original system. The data used is labelled in the sub-caption.}
\label{fig:norm_diff}
\end{figure}

\subsection{Reduction of dilation dimension with maximum matching in bipartite graph}
\label{sec.exp.matching}
In this section we illustrate how the dimension $n_C$ of the cyclic dilation obtained from the maximum matching program in bipartite graphs discussed in in Remark~\ref{rmk:graphmatching} 
can yield reservoir sizes drastically lower than the theoretical upper bound
given by the approximating full-cycle permutation $P$:
\[
   n_1 = 2\cdot \ell_0 \cdot \left(k+1\right)
    > \left\lceil 2\cdot \frac{\pi}{\operatorname{arccos}\left(1-\frac{\delta^2}{2}\right)} \cdot \left(k+1\right) \right\rceil.
\]

We generate $10$ orthogonal matrices $U$ uniformly randomly for each initial dimension $n \in \left\{20,40,\ldots,140,160\right\}$. We perform cyclic dilation as described in the previous section and compare it against the theoretical upper bound. Notice that the $y$-axis is in log scale. 
The dimension $n_C$ is significantly lower than the theoretical upper bound, reaching $\approx 300-400$ units for initial reservoirs of size $80-160$, which is well within possibilities of hardware implementations of such reservoirs. 

\begin{figure}[ht!]
\centering
  \includegraphics[width=0.99\linewidth]{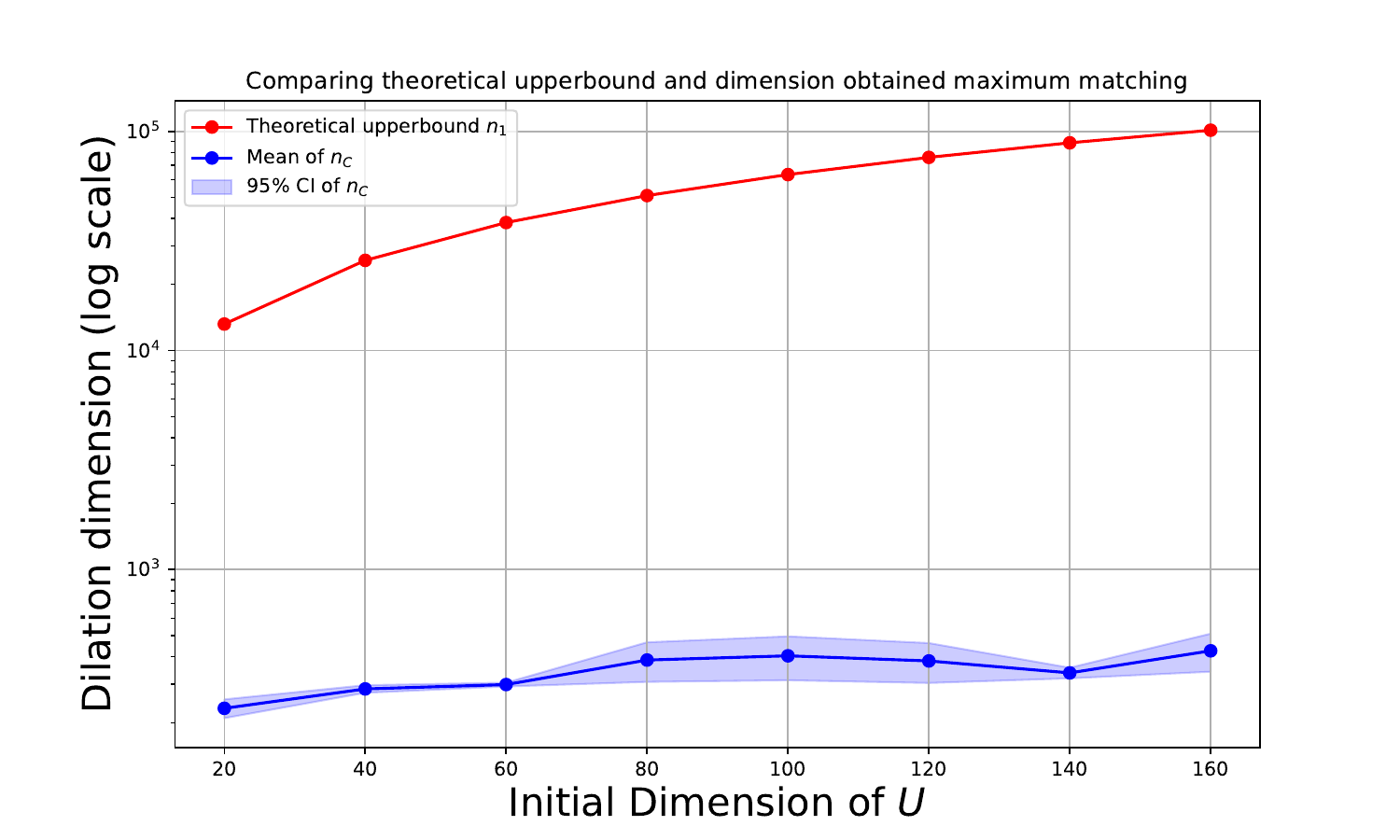}
  \caption{
Theoretical upperbound v.s. dimension from maximum matching of bipartite graph.}
  \label{fig:mmbpvsupper}
\end{figure}%

\section{Conclusion}

In this paper, we rigorously demonstrated the universality of Simple Cycle Reservoir (SCR) in the real domain, adopting the strategy from \cite{li2023simple}. Specifically, we proved that SCRs are universal approximations for any real-valued unrestricted linear reservoir system and any real-valued time-invariant fading memory filter over uniformly bounded input streams.

To achieve this, we constrained our approach to the real domain throughout the approximation pipeline. We performed cyclic dilation of orthogonally dilated coupling matrices by completing the set of roots of unity in the canonical form of orthogonal matrices, rather than using the eigendecomposition method involving unitary matrices. This ensured that all approximant systems remained in the real domain under orthogonal similarity.

We facilitated the completion of roots of unity by utilizing a maximum matching program in bipartite graphs, enabling a tighter dimension expansion of the approximation system. This method ensured efficient and effective expansion to achieve the desired approximation accuracy.

The fully constructive nature of our results is a crucial step towards the physical implementations of reservoir computing \cite{Appeltant2011InformationPU, NTT_cyclic_RC, bienstman2017, Abe2024May}.

\clearpage

\newcommand{\etalchar}[1]{$^{#1}$}

\clearpage

\section{Appendix: proofs of results}
\subsection{Proof of Theorem \ref{thm.dilation}}
\label{pf:thm.dilation}
\begin{proof} The uniform boundedness of input stream and contractiveness of $W$ imply that the state space $X\subseteq \mathbb{R}^n$ is closed and bounded, hence compact. The continuous readout map $h$ is therefore uniformly continuous on the state space $X$. 

By the uniform continuity of $h$, for {any} $\epsilon >0$, there exists $\delta>0$ such that  {for any $\x,\x' \in X$ with $\|\Bx-\Bx'\|<\delta$, we have $\|h(\Bx)-h(\Bx')\|<\epsilon$}. Let $\lambda = \|W\|$ and let $M$ denote the uniform bound of  $\{\u_t\}$ such that $\norm{\u_t}\leq M$ for all $t$. Since $\lambda <1$, we can choose $N$, such that:
 {
\[
2 M \|V\| \sum_{t>N} \|W\|^t  = 2 M \|V\|  \frac{\lambda^{N+1}}{1-\lambda} < \delta.
\]
}
Let $W_1 = W/\lambda$ and $n'=(N+1) \cdot n$. We have $\|W_1\|=1$ and therefore by Egerv\'{a}ry's dilation, there exists a orthogonal $n'\times n'$ matrix $U$  such that for all $1\leq k\leq N$, we have: 
\begin{align*}
    W_1^k = J^\top U^k J , 
\end{align*}
where $J:\mathbb{R}^{n} \hookrightarrow \mathbb{R}^{n'}$ is the canonical embedding of $\mathbb{R}^n$ onto the first $n$-coordinates of $\mathbb{R}^{n'}$. Let $W'=\lambda U$, then it follows immediately that: 
\[
W^k = \lambda^k W_1^k = J^\top \left(\lambda U\right)^k J = J^\top \left(W'\right)^k J.
\]
Define an $n'\times n$ matrix by:
\begin{equation}
\label{eqn.unitary.v'}
V'=\begin{bmatrix}
V  \\
0
\end{bmatrix},
\end{equation}
and the map $h':\mathbb{R}^{n'} \to \mathbb{R}^d$ given by:
\begin{align}
\label{eqn.unitary.h'}
& h'(x_1, x_2, \cdots,x_n,\cdots,x_{n'}) \\
& := h(x_1, x_2, \cdots, x_n) = h\circ J^\top(x_1, x_2, \cdots,x_n,\cdots,x_{n'}) \nonumber 
\end{align}

We now show that the reservoir system $R'=(W', V', h')$ is $\epsilon$-close to $R=(W,V,h)$.

For any input stream 
$\{\u_t\}_{t\in\mathbb{Z}_-}$, consider the states under the reservoir systems $R$ and $R'$ given by: 
\begin{align}
\label{eqn.unitary.state}
{\Bx_t} &= \sum_{k\geq 0} W^k V \u_{t-k} \nonumber \\
{\Bx_t'} &= \sum_{k\geq 0} \left(W'\right)^k V' \u_{t-k} 
\end{align}

For each $k \geq 0$, we denote the upper left $n\times n$ block of $\left(W'\right)^k$ by $A_k$. In other words:
\[\left(W'\right)^k = \begin{bmatrix}
A_k & * \\
* & *
\end{bmatrix}.\]
This splits into two cases. For each $0\leq k\leq N$, we have $A_k=W^k$ by construction of $W'$. Otherwise, for $k>N$, the power $k$ is beyond the dilation power bound and we no longer have $A_k=W^k$ in general. Nevertheless, since $A_k$ is a submatrix of $\left(W'\right)^k$, it's operator norm is bounded from above:
\[\norm{A_k}\leq \norm{W^k}\leq \norm{\left(W'\right)^k}\leq \norm{W'}^k = \lambda^k.\]

By Equation \eqref{eqn.unitary.v'}, we have $V' \u_{t-k}=\begin{bmatrix}
V \u_{t-k} \\
0
\end{bmatrix}$ and the state ${\Bx_t'}$ of $R'$ from Equation \eqref{eqn.unitary.state} thus becomes: 

\begin{align*}
{\Bx_t'} &= \sum_{k\geq 0} \left(W'\right)^k V' \u_{t-k} \\
&= \sum_{k\geq 0} \begin{bmatrix}
A_k & * \\
* & *
\end{bmatrix}  \begin{bmatrix}
V \u_{t-k} \\
0
\end{bmatrix}\\
&= \sum_{k=0}^N \begin{bmatrix}
W^k & * \\
* & *
\end{bmatrix}  \begin{bmatrix}
V \u_{t-k} \\
0
\end{bmatrix}  + \sum_{k>N} \begin{bmatrix}
A_k & * \\
* & *
\end{bmatrix}  \begin{bmatrix}
V \u_{t-k} \\
0
\end{bmatrix} \\
&= \begin{bmatrix}
\sum_{k=0}^N W^k V \u_{t-k} \\
*
\end{bmatrix} + \begin{bmatrix}
\sum_{k>N} A_k V \u_{t-k} \\
*
\end{bmatrix}.
\end{align*}

Let $J^\top({\Bx_t'})$ be the first $n$-coordinates of ${\Bx_t'}$. 
We have 

\[J^\top({\Bx_t'})= \sum_{k=0}^N W^k V \u_{t-k}+\sum_{k>N} A_k V \u_{t-k}.\]
Comparing with
\[{\Bx_t} = \sum_{k\geq 0} W^k V \u_{t-k} = \sum_{k=0}^N W^k V \u_{t-k} + \sum_{k>N} W^k V \u_{t-k},\]
it follows immediately that:
\begin{align*}
\norm{J^\top({\Bx_t'})-{\Bx_t}} &= \norm{0+\sum_{k>N} \left(A_k - W^k\right) V \u_{t - k}} \\
& \leq 
\sum_{k>N} \left(\norm{A_k}+\norm{W^k}\right) \norm{V} M
\end{align*}
Notice we have $\|W^k\|\leq \|W\|^k=\lambda^k$ and we also showed $\norm{A_k}\leq \lambda^k$, and therefore:

\[\|J^\top({\Bx_t'})-{\Bx_t}\|\leq  \sum_{k>N} 2\lambda^k \|V\| M < \delta \]
By Equation \eqref{eqn.unitary.h'} $h'({\Bx_t})=h(J^\top({\Bx_t}))$ and by uniform continuity of $h$ we have:
\[\|{\By_t}-{\By_t'}\| = \|h({\Bx_t}) - h(J^\top({\Bx_t}))\|<\epsilon\]
This finishes the proof. 
\end{proof}

\subsection{Proof of Theorem \ref{thm.to.permutation}}
\label{pf:thm.topermutation}
\begin{proof} Let $\epsilon >0$ be arbitrary. By the proof of Theorem \ref{thm.dilation},
the state space $X$ is compact and we can choose $\delta$ such that $\|\Bx-\Bx'\|<\delta$ implies $\|h(\Bx)-h(\Bx')\|<\epsilon$. Let $M:=\sup \|\u_t\|<\infty$, since $\lambda < 1$ we can pick $N>0$ such that 
\begin{align}
\label{eqn.halfdelta_end}
2 M \|V\|  \sum_{k>N} \lambda^k  < \frac{\delta}{2}.
\end{align}
Once we fix such an $N$, pick $\delta_0>0$ such that
\begin{align}
\label{eqn.halfdelta_front}
M \|V\| \sum_{k=0}^N ((\lambda+\delta_0)^k-\lambda^k)  < \frac{\delta}{2}.
\end{align}
Such a $\delta_0$ exists because the left-hand side is a finite sum that is continuous in $\delta_0$ and tends to $0$ as $\delta_0\to 0$.
According to Theorem~\ref{thm.perturb.unitary}, there exists a $n_1\times n_1$ full-cycle permutation matrix $P$, an orthogonal matrix $S$, and an orthogonal matrix $D$ such that:
\[
\norm{S^\top P S-\begin{bmatrix} U & 0 \\ 0 & D\end{bmatrix}}<\frac{1}{\lambda}\min\{\delta,\delta_0\}.
\]

Let $A = S^\top P S$ and let $Q_n:\mathbb{R}^{n_1} \hookrightarrow \mathbb{R}^n$ be the canonical projection onto the first $n$ coordinates. Consider the reservoir systems $R_0 :=(W_0, V_0, h_0)$ and $R_1 :=(W_1, V_0, h_0)$ defined by the following:
\begin{align*}
    W_0 = \lambda \begin{bmatrix} U & 0 \\ 0 & D\end{bmatrix} , & \quad V_0 = \begin{bmatrix} V \\ 0 \end{bmatrix} \\
    W_1 = \lambda A,&\quad  h_0(\Bx) = h(Q_n(\Bx)).
\end{align*}
Notice that the choice of $A$ ensures that $\norm{W_1 - W_0} <  \min\{\delta, \delta_0\}$.

The rest of the proof is outlined as follows: We first show that $R_0$ is equivalent to $R$, and then prove that $R_1$ is $\epsilon$-close to $R_0$. By Theorem \ref{thm.perturb.unitary}, $A$ is orthogonally equivalent to a full-cycle permutation matrix, and the desired results follow from \cite[Proposition 12]{li2023simple}. We now flesh out the above outline: 
We first establish that $R_0$ is equivalent to $R$. For any input stream $\{\u_t\}_{t\in\mathbb{Z}_-}$, the solution to $R_0$ is given by
\begin{align*}
    {\By}_t^{(0)} &= h_0\left(\sum_{k\geq 0} W_0^k V_0 \u_{t-k}\right) \\
    &= h\left(Q_n \left(\sum_{k\geq 0} \begin{bmatrix} (\lambda U)^k & 0 \\ 0 & (\lambda D)^k \end{bmatrix} \begin{bmatrix} V \\ 0 \end{bmatrix} \u_{t-k}\right)\right) \\
    &= h\left(Q_n \left(\begin{bmatrix} \sum_{k\geq 0} W^k V \u_{t-k} \\0 \end{bmatrix}\right)\right) \\
    &= h\left(\sum_{k\geq 0} W^k V \u_{t-k}\right).
\end{align*}
This is precisely the solution to $R$. 

We now show that $R_1$ is $\epsilon$-close to $R_0$. 

First, we observe that since $Q_n$ is a projection onto the first $n$-coordinates, it has operator norm $\norm{Q_n}=1$ and thus  {whenever  $\norm{\Bx-\Bx'}<\delta$, $\Bx,\Bx'\in\mathbb{R}^{n_1}$, we have} $\norm{Q_n \Bx- Q_n \Bx'}<\delta$ and thus $\|h(Q_n \Bx)-h(Q_n \Bx')\|<\epsilon$. Therefore it suffices to prove that for any input $\{\u_t\}$, the solution to $R_0$, given by
\[{\Bx}_t^{(0)} = \sum_{k\geq 0} W_0^k V_0 \u_{t-k},\]
is within $\delta$ to the solution to $R_1$, given by
\[{\Bx}_t^{(1)} = \sum_{k\geq 0} W_1^k V_0 \u_{t-k}.\]
By construction $V_0=\begin{bmatrix} V \\ 0\end{bmatrix}$ has $\|V_0\|=\|V\|$, hence:
\begin{align}
\label{eqn.state_r0r1_diff}
 \norm{{\Bx}_t^{(0)} - {\Bx}_t^{(1)}} &= \norm{\sum_{k\geq 0} (W_0^k-W_1^k) V_0 \u_{t-k}} \nonumber \\
&\leq \sum_{k\geq 0}  \norm{(W_0^k-W_1'^k)} \norm{V_0} M \nonumber \\
&= \sum_{k=0}^N  \norm{\left(W_0^k-W_1^k\right)} \norm{V} M \nonumber \\
& \quad + \sum_{j>N}  \norm{\left(W_0^k-W_1^k\right)} \norm{V} M.
\end{align}

Consider $\Delta=W_0-W_1$, we then have $\|\Delta\|<\delta_0$ and for each $0\leq j\leq N$, $W_0^j-W_1^j=(W_1+\Delta)^j - W_1^j$. Expanding $(W_1+\Delta)^j$, we get a summation of $2^j$ terms of the form $\prod_{i=1}^j X_i$, where each $X_i=W_1$ or $\Delta$.  For $s = 0,\ldots , j$, each of the $2^j$ terms has norm $\norm{\prod_{i=1}^j X_i}\leq \|W_1\|^{j-s} \|\Delta\|^{s}$ if there are $s$ copies of $\Delta$ among $X_i$. 
 {Removing the term $W_1^j$ from $(W_1+\Delta)^j$ results in all the remaining terms containing at least one copy of $\Delta$.} We thus arrive at: 

\begin{align}
    \norm{W_0^j - W_1^j} &= \norm{\left(W_1+\Delta\right)^j - W_1^j} 
    \nonumber \\
    &\leq \sum_{s=1}^j \binom{j}{s} \norm{W_1}^{j-s} \norm{\Delta}^{s} 
    \nonumber \\
   &\leq \sum_{s=1}^j \binom{j}{s} \lambda^{j-s} \delta_0^{s} = (\lambda+\delta_0)^j - \lambda^j.
\label{eq:bound}
\end{align}
Combining the above with Equation \eqref{eqn.halfdelta_front}, we have:
\[
M \|V\| \sum_{j=0}^N  \|(W_0^j-W_1^j)\| \leq 
M \|V\| \sum_{j=0}^N ((\lambda+\delta_0)^j - \lambda^j) < \frac{\delta}{2}. 
\]
On the other hand by Equation \eqref{eqn.halfdelta_end}, we obtain:
\begin{eqnarray*}
M \|V\| \sum_{j>N}  \|(W_0^j-W_1^j)\|  
&\leq& 
M \|V\| \sum_{j>N}  (\|W_0\|^j+\|W_1\|^j)\\
&\leq& 2 M \|V\| \sum_{j>N}  \lambda^j\\
&<&\frac{\delta}{2}.
\end{eqnarray*}
With the two inequalities above, Equation \eqref{eqn.state_r0r1_diff} thus becomes:
\[ \|{\Bx}_t^{(0)} - {\Bx}_t^{(1)} \| <\delta.\]
Uniform continuity of $h$ implies $\norm{h({\Bx}_t^{(0)}) - h({\Bx}_t^{(1)})}<\epsilon$, proving $R_1$ is $\epsilon$-close to $R_0$. 
Finally, by Theorem~\ref{thm.perturb.unitary}, $A$ is orthogonally equivalent to a full-cycle permutation matrix $P$, i.e. there exists orthogonal matrix $S$ such that $S^\top A S = P$. By \cite[Proposition 12]{li2023simple}, we obtain a reservoir system $R_c=(W_c, V_c, h_c)$ with $W_c=\lambda P$, such that $R_c$ is equivalent to $R_1$, which is in turn $\epsilon$-close to $R_0$. Since the original reservoir system $R$ is equivalent to $R_0$, $R$ is therefore $\epsilon$-close to $R_c$, as desired.
\end{proof}


\begin{thebibliography}{ASdS{\etalchar{+}}11}

\bibitem[ANH{\etalchar{+}}24]{Abe2024May}
Yuki Abe, Kazuki Nakada, Naruki Hagiwara, Eiji Suzuki, Keita Suda, Shin-ichiro
  Mochizuki, Yukio Terasaki, Tomoyuki Sasaki, and Tetsuya Asai.
\newblock {Highly-integrable analogue reservoir circuits based on a simple
  cycle architecture}.
\newblock {\em Sci. Rep.}, 14(10966):1--10, May 2024.

\bibitem[ASdS{\etalchar{+}}11]{Appeltant2011InformationPU}
Lennert Appeltant, Miguel~C. Soriano, Guy~Van der Sande, Jan Danckaert, Serge
  Massar, Joni Dambre, Benjamin Schrauwen, Claudio~R. Mirasso, and Ingo
  Fischer.
\newblock Information processing using a single dynamical node as complex
  system.
\newblock {\em Nature Communications}, 2, 2011.

\bibitem[CSK{\etalchar{+}}18]{bienstman2017}
Florian Denis-Le Coarer, Marc Sciamanna, Andrew Katumba, Matthias Freiberger,
  Joni Dambre, Peter Bienstman, and Damien Rontani.
\newblock {All-Optical Reservoir Computing on a Photonic Chip Using
  Silicon-Based Ring Resonators}.
\newblock {\em {IEEE Journal of Selected Topics in Quantum Electronics}},
  24(6):1 -- 8, November 2018.

\bibitem[GO18a]{Grigoryeva2018}
L.~Grigoryeva and J.-P. Ortega.
\newblock Universal discrete-time reservoir computers with stochastic inputs
  and linear readouts using non-homogeneous state-affine systems.
\newblock {\em J. Mach. Learn. Res.}, 19(1):892--931, January 2018.

\bibitem[GO18b]{grigoryeva2018echo}
Lyudmila Grigoryeva and Juan-Pablo Ortega.
\newblock Echo state networks are universal.
\newblock {\em Neural Networks}, 108:495--508, 2018.

\bibitem[GO19]{gonon2019reservoir}
Lukas Gonon and Juan-Pablo Ortega.
\newblock Reservoir computing universality with stochastic inputs.
\newblock {\em IEEE transactions on neural networks and learning systems},
  31(1):100--112, 2019.

\bibitem[Hal50]{Halmos1950}
Paul~R. Halmos.
\newblock Normal dilations and extensions of operators.
\newblock {\em Summa Brasil. Math.}, 2:125--134, 1950.

\bibitem[Jae01]{Jaeger2001}
H.~Jaeger.
\newblock The "echo state" approach to analysing and training recurrent neural
  networks.
\newblock Technical report gmd report 148, German National Research Center for
  Information Technology, 2001.

\bibitem[Jae02a]{Jaeger2002}
H.~Jaeger.
\newblock Short term memory in echo state networks.
\newblock Technical report gmd report 152, German National Research Center for
  Information Technology, 2002.

\bibitem[Jae02b]{jaeger2002a}
H.~Jaeger.
\newblock A tutorial on training recurrent neural networks, covering bppt,
  rtrl, ekf and the "echo state network" approach.
\newblock Technical report gmd report 159, German National Research Center for
  Information Technology, 2002.

\bibitem[JH04]{Jaeger2004}
H.~Jaeger and H.~Haas.
\newblock Harnessing nonlinearity: predicting chaotic systems and saving energy
  in wireless telecommunication.
\newblock {\em Science}, 304:78--80, 2004.

\bibitem[LFT24]{li2023simple}
Boyu Li, Robert~Simon Fong, and Peter Tino.
\newblock {Simple Cycle Reservoirs are Universal}.
\newblock {\em Journal of Machine Learning Research}, 25(158):1--28, 2024.

\bibitem[LJ09]{Lukoservicius2009}
M.~Lukosevicius and H.~Jaeger.
\newblock Reservoir computing approaches to recurrent neural network training.
\newblock {\em Computer Science Review}, 3(3):127--149, 2009.

\bibitem[MNM02]{Maass2002}
W.~Maass, T.~Natschlager, and H.~Markram.
\newblock Real-time computing without stable states: a new framework for neural
  computation based on perturbations.
\newblock {\em Neural Computation}, 14(11):2531--2560, 2002.

\bibitem[NTH21]{NTT_cyclic_RC}
Mitsumasa {Nakajima}, Kenji {Tanaka}, and Toshikazu {Hashimoto}.
\newblock {Scalable reservoir computing on coherent linear photonic processor}.
\newblock {\em Communications Physics}, 4(1):20, December 2021.

\bibitem[Pau02]{PaulsenBook}
Vern Paulsen.
\newblock {\em Completely bounded maps and operator algebras}, volume~78 of
  {\em Cambridge Studies in Advanced Mathematics}.
\newblock Cambridge University Press, Cambridge, 2002.

\bibitem[RT10]{rodan2010minimum}
Ali Rodan and Peter Ti\v{n}o.
\newblock Minimum complexity echo state network.
\newblock {\em IEEE transactions on neural networks}, 22(1):131--144, 2010.

\bibitem[TD01]{Tino2001}
P.~Ti\v{n}o and G.~Dorffner.
\newblock Predicting the future of discrete sequences from fractal
  representations of the past.
\newblock {\em Machine Learning}, 45(2):187--218, 2001.

\bibitem[ZZP{\etalchar{+}}20]{zhou2021informer}
Haoyi Zhou, Shanghang Zhang, Jieqi Peng, Shuai Zhang, Jianxin Li, Hui Xiong,
  and Wan Zhang.
\newblock Informer: Beyond efficient transformer for long sequence time-series
  forecasting.
\newblock In {\em AAAI Conference on Artificial Intelligence}, 2020.

\end{thebibliography}
\end{document}